\documentclass{article}

\usepackage{microtype}
\usepackage{graphicx}
\usepackage{subfigure}
\usepackage{multirow}
\usepackage{bm}
\usepackage{caption}
\usepackage{booktabs} % for professional tables

\usepackage{hyperref}

\newcommand{\model}{GRED}

\def\mI{{\bm{I}}}

\def\mW{{\bm{W}}}

\def\vs{{\bm{s}}}

\def\vx{{\bm{x}}}

\def\vz{{\bm{z}}}

\newcommand{\R}{\mathbb{R}}
\newcommand{\C}{\mathbb{C}}

\newcommand{\diag}{\mathrm{diag}}

\def\mLambda{{\bm{\Lambda}}}
\def\vlambda{{\bm{\lambda}}}

\usepackage[accepted]{icml2024}

\definecolor{darkorange}{RGB}{190,95,50}
\definecolor{darkpink}{RGB}{180,50,120}
\definecolor{darkblue}{RGB}{40,80,120}

\usepackage{amsmath}
\usepackage{amssymb}
\usepackage{mathtools}
\usepackage{amsthm}

\usepackage[capitalize,noabbrev]{cleveref}

\theoremstyle{plain}
\newtheorem{theorem}{Theorem}[section]

\newtheorem{corollary}[theorem]{Corollary}
\theoremstyle{definition}

\theoremstyle{remark}

\usepackage[textsize=tiny]{todonotes}

\icmltitlerunning{Recurrent Distance Filtering for Graph Representation Learning}

\begin{document}

\twocolumn[
\icmltitle{Recurrent Distance Filtering for Graph Representation Learning}



\icmlsetsymbol{equal}{*}

\begin{icmlauthorlist}
\icmlauthor{Yuhui Ding}{ethz}
\icmlauthor{Antonio Orvieto}{mpi}
\icmlauthor{Bobby He}{ethz}
\icmlauthor{Thomas Hofmann}{ethz}

%\icmlauthor{Firstname8 Lastname8}{sch}
%\icmlauthor{Firstname8 Lastname8}{yyy,comp}
\end{icmlauthorlist}

\icmlaffiliation{ethz}{Department of Computer Science, ETH Zürich}
\icmlaffiliation{mpi}{ELLIS Institute Tübingen, MPI-IS, Tübingen AI Center}

\icmlcorrespondingauthor{Yuhui Ding}{yuhui.ding@inf.ethz.ch}

\icmlkeywords{Machine Learning, ICML, Graph Neural Networks}

\vskip 0.3in
]



\printAffiliationsAndNotice{}  % leave blank if no need to mention equal contribution

\begin{abstract}
Graph neural networks based on iterative one-hop message passing have been shown to struggle in harnessing the information from distant nodes effectively.
Conversely, graph transformers allow each node to attend to all other nodes directly, but lack graph inductive bias and have to rely on ad-hoc positional encoding.
In this paper, we propose a new architecture to reconcile these challenges. Our approach stems from the recent breakthroughs in long-range modeling provided by deep state-space models: for a given target node, our model aggregates other nodes by their shortest distances to the target and uses a linear RNN to encode the sequence of hop representations. The linear RNN is parameterized in a particular diagonal form for stable long-range signal propagation and is theoretically expressive enough to encode the neighborhood hierarchy.
With no need for positional encoding, we empirically show that the performance of our model is comparable to or better than that of state-of-the-art graph transformers on various benchmarks, with a significantly reduced computational cost.
Our code is open-source at \url{https://github.com/skeletondyh/GRED}. 
\end{abstract}

\section{Introduction}
\label{sec:intro}
Graphs are ubiquitous for representing complex interactions between individual entities,
such as in social networks~\cite{tang2009social}, recommender systems~\cite{ying2018graph} and molecules~\cite{gilmer2017neural},
and have thus attracted a lot of interest from researchers seeking to apply deep learning to graph data.
Message passing neural networks (MPNNs)~\cite{gilmer2017neural} have been the dominant approach in this field.
These models iteratively update the representation of a target node by aggregating the representations of its neighbors. 
Despite progress in semi-supervised node classification tasks~\cite{kipf2016semi, velivckovic2017graph},
MPNNs have been shown to have difficulty in effectively harnessing the information of distant nodes~\cite{alon2020bottleneck, dwivedi2022long}.
To reach a node that is $k$ hops away from the target node,
an MPNN needs at least $k$ layers.
As a result, the receptive field for the target node grows exponentially with $k$,
including many duplicates of nodes that are close to the target node.
The information from such an exponentially growing receptive field is compressed into a fixed-size representation,
making it insensitive to the signals from distant nodes (a.k.a. over-squashing~\cite{topping2021understanding, di2023over}).
This limitation may hinder the application of MPNNs to tasks that require long-range reasoning.

\begin{figure}[t]
    \centering
    \includegraphics[width=0.8\linewidth]{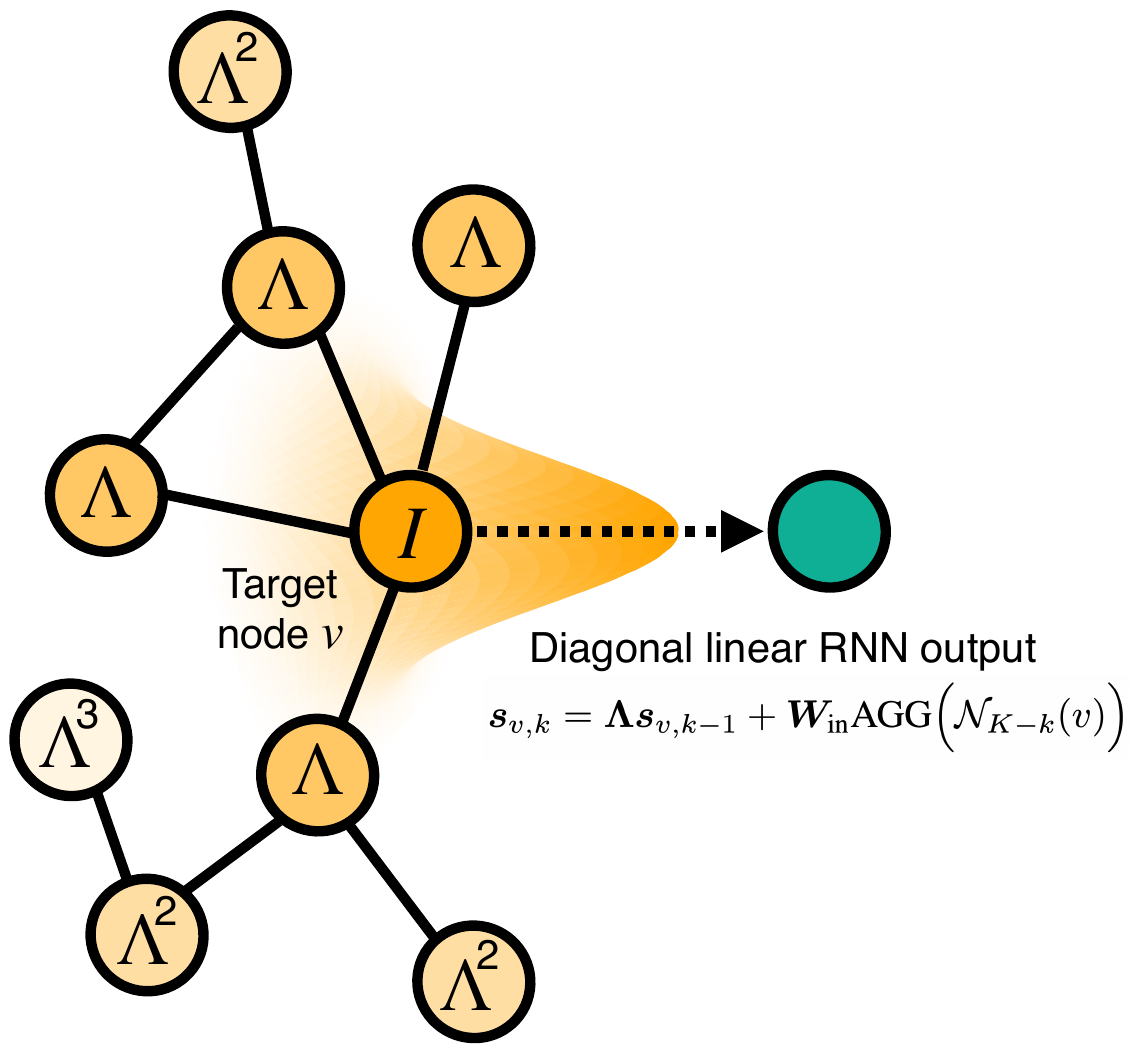}
    \caption{Illustration of the filtering effect on the neighborhood, induced by the linear RNN. The filter weight is determined by the eigenvalues $\mLambda$ of the transition matrix and the shortest distance to the target node. We expand on this in Section~\ref{sec:arch}.}
    \label{fig:filter}
    \vspace{-4mm}
\end{figure}

Inspired by the success of attention-based transformer architectures in modeling natural languages~\cite{vaswani2017attention, devlin2018bert} and images~\cite{dosovitskiy2020image},
several recent works have adapted transformers for graph representation learning to address the aforementioned issue~\cite{ying2021transformers, kim2022pure, chen2022structure, ma2023graph}.
Graph transformers allow each node to attend to all other nodes directly through a global attention mechanism, and therefore make the information flow between distant nodes easier.
However, a naive global attention mechanism alone doesn't encode any structural information about the underlying graph.
Hence, state-of-the-art graph transformers rely on ad hoc positional encoding (e.g., eigenvectors of the graph Laplacian) as extra features to incorporate
the graph inductive bias.
There is no consensus yet on the optimal type of positional encoding.
Which positional encoding to use and its associated hyper-parameters need to be tuned carefully~\cite{rampavsek2022recipe}.
Besides, while graph transformers have empirically shown improvement on some graph benchmarks compared with classical MPNNs,
the former are much more computationally expensive~\cite{dwivedi2022long}.

Captivated by the above challenges and the need for powerful, theoretically sound and computationally efficient approaches to graph representation learning,
we propose a new model, %\bobby{include name to make more direct, and put figure at top of page 2}
Graph Recurrent Encoding by Distance (GRED).
Each layer of our model consists of a permutation-invariant neural network~\cite{zaheer2017deep} and a linear recurrent neural network~\cite{orvieto2023resurrecting}
that is parameterized in a particular diagonal form following the recent advances in state space models~\cite{gu2022parameterization, smith2023simplified}.
To generate the representation for a target node,
our model categorizes all other nodes into multiple sets according to their shortest distances to the target node.
The permutation-invariant neural network generates a representation for each set of nodes that share the same shortest distance to the target node, and then the linear recurrent neural network encodes the sequence of the set representations,
starting from the set with the maximum shortest distance and ending at the target node itself.
Since the order of the sequence is naturally encoded by the recurrent neural network,
our model can encode the neighborhood hierarchy of the target node without the need for positional encoding. 
The architecture of \model~is illustrated in Figure~\ref{fig:architecture}.

The diagonal parameterization of the linear RNN~\cite{orvieto2023resurrecting} has been shown to 
make long-range signal propagation more stable than a vanilla RNN, and enables our model to effectively harness the information of distant nodes.
More specifically, it enables our model to directly learn the eigenvalues of the transition matrix,
which control how fast the signals from distant nodes decay as they propagate towards the target node (see \cref{fig:filter} for an illustration),
and at the same time allows efficient computation with parallel scans.
Furthermore,
while the use of a linear recurrent neural network is motivated by long-range signal propagation,
we theoretically prove its expressive power in terms of injective functions over sequences,
which is of independent interest,
and based on that we conclude that our model is more expressive than 1-WL~\cite{xu2018powerful}.
We evaluate our model on a series of graph benchmarks to support its efficacy.
The performance of our model is significantly better than that of MPNNs, 
and is comparable to or better than that of state-of-the-art graph transformers while
requiring no positional encoding and significantly reducing computation time.

To summarize,
the main contributions of our paper are as follows:

\begin{itemize}
    \item We propose a principled new model for graph representation learning that can effectively and efficiently harness the information of distant nodes.
    The architecture is composed of permutation-invariant neural networks and linear recurrent neural networks with diagonal parameterization.
    \item We theoretically prove that a linear recurrent neural network is able to express an injective mapping over sequences, which makes our architecture more expressive than 1-WL.
    \item Without the need for positional encoding, our model has achieved strong empirical performance on multiple widely used graph benchmarks, which is comparable to or better than that of state-of-the-art graph transformers, with higher training efficiency.
\end{itemize}

\section{Related Work}
We review below the literature on expanding MPNN's receptive field, including multi-hop MPNNs and graph transformers,
as well as current trends in recurrent models for long-range reasoning on sequential data.

\paragraph{Multi-hop MPNNs.}
Multi-hop MPNNs leverage the information of multiple hops for each layer.
Among existing works,
MixHop~\cite{abu2019mixhop} uses powers of the normalized adjacency matrix to access $k$-hop nodes.
$k$-hop GNN~\cite{nikolentzos2020k} iteratively applies MLPs to combine two consecutive hops and propagates information towards the target node.
\citeauthor{feng2022how}~\yrcite{feng2022how}~theoretically analyze the expressive power of general $k$-hop MPNNs and enhance it with subgraph information.
These works proved that higher-hop information can improve the expressiveness of MPNNs,
but they didn't address how to preserve long-range information during propagation as we do.
SPN~\cite{abboud2022shortest} is shown to alleviate over-squashing empirically.
It first aggregates neighbors of the same hop but simply uses weighted summation to combine hop representations, which cannot guarantee the expressiveness of the model.
On the contrary,
we prove that our model,
capable of modeling long-range dependency,
is also theoretically expressive.
PathNN~\cite{michel2023path} encodes each individual path that emanates from a node and aggregates these paths to compute the node representation.
DRew~\cite{gutteridge2023drew} gradually aggregates more hops at each layer and allows skip connections between different nodes.

\begin{figure*}[t]
\centering
\includegraphics[width=0.915\textwidth]{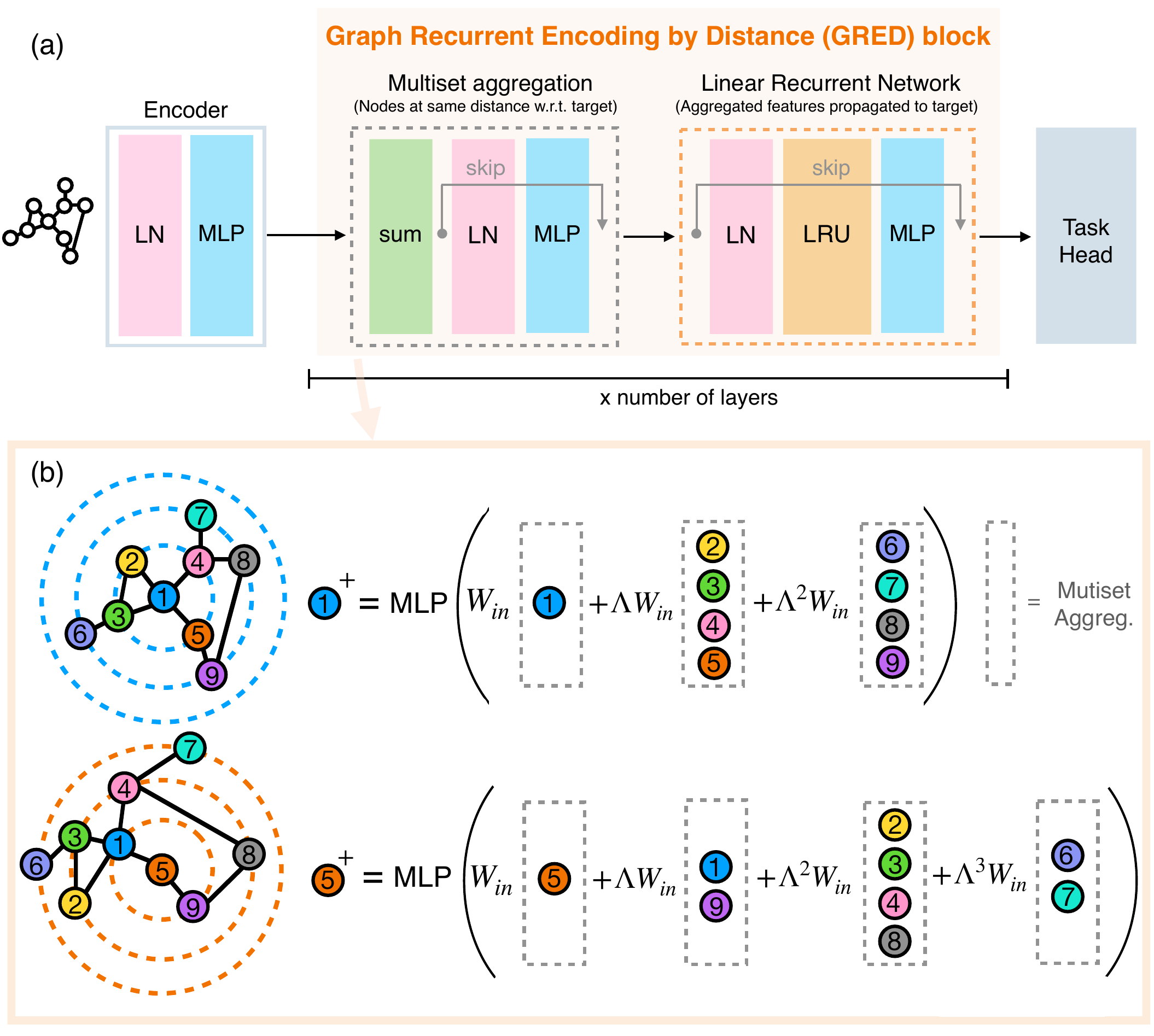}
\caption{(a) Sketch of the architecture. MLPs and Layer Normalization operate independently at each node or aggregated multiset. Information of the distant nodes is propagated to the target node through a linear RNN -- specifically an LRU~\cite{orvieto2023resurrecting}. (b) Depiction of the \model~layer operation for two different target nodes. The gray rectangular boxes indicate the application of multiset aggregation. Finally, the new representation for the target node is computed from the RNN output through an MLP.}
\label{fig:architecture}
\vspace{-2mm}
\end{figure*}

\paragraph{Graph transformers.}
Graph transformers~\cite{ying2021transformers, wu2021representing, chen2022structure, rampavsek2022recipe, zhang2023rethinking, ma2023graph} have recently attracted a lot of interest because the global attention mechanism allows each node to directly attend to all other nodes.
To bake in the graph structural information, graph transformers typically use positional encoding~\cite{li2020distance, dwivedi2021graph} as extra features.
More specifically,
Graphormer~\cite{ying2021transformers} adds learnable biases to the attention matrix for different shortest distances.
However, the sequential order of hops is not encoded into the model, and Graphormer still needs node degrees to augment node features.
SAT~\cite{chen2022structure} and GraphTrans~\cite{wu2021representing} stack message passing layers and self-attention layers together to obtain local information before the global attention.
\citeauthor{rampavsek2022recipe}~\yrcite{rampavsek2022recipe}~empirically compare different configurations of positional encoding, message passing and global attention.
\citeauthor{zhang2023rethinking}~\yrcite{zhang2023rethinking}~suggest the use of resistance distance as relative positional encoding. 
\citeauthor{ma2023graph}~\yrcite{ma2023graph}~use learnable positional encoding initialized with random walk probabilities.
\citeauthor{he2023generalization}~\yrcite{he2023generalization}~use MPNNs to encode graph patches generated by a graph clustering algorithm and apply MLP-Mixer~\cite{tolstikhin2021mlp}/ViT~\cite{dosovitskiy2020image}
to patch embeddings, but require node/patch positional encoding and selecting the number of patches.

\paragraph{State space models and linear RNNs.} Efficient processing of long sequences is one of the paramount challenges in contemporary deep learning. Attention-based transformers~\cite{vaswani2017attention} provide a scalable approach to sequential modeling but suffer from \textit{quadratically increasing inference/memory complexity} as the sequence length grows. 
While many approaches exist to alleviate this issue, like efficient memory management~\cite{dao2022flashattention,dao2023flashattention} and architectural modifications~\cite{wang2020linformer, kitaev2020reformer, child2019generating, beltagy2020longformer, wu2020lite}, the sequence length in modern large language models is usually kept to $2k/4k$ tokens for this reason~(e.g. Llama2~\cite{touvron2023llama}). On top of high inference and memory cost, the attention mechanism often does not provide the correct \textit{inductive bias} for long-range reasoning beyond text~\cite{tay2020long}. 
%Indeed, most transformers~(including long-range/sparse variants, reduced complexity variants, or variants with other tricks) are often found to perform poorly in discovering long-range dependencies in data~\cite{tay2020long}. 
Due to the issues outlined above, the community has witnessed the rise of innovative \textit{recurrent} alternatives to the attention mechanism, named state space models~(SSMs). The first SSM, S4, was introduced by \citeauthor{gu2021efficiently}~\yrcite{gu2021efficiently}~based on the theory of polynomial signal approximation~\cite{gu2020hippo,gu2023train} and significantly surpassed all modern transformer variants on the challenging long-range benchmark~\cite{tay2020long}.
Since then, a plethora of variants have been proposed~\cite{hasani2022liquid,gupta2022diagonal,smith2023simplified,peng2023rwkv}.
%These models achieve remarkable performance, surpassing all modern attention-based transformer variants by an average $20\%$ accuracy on challenging sequence classification tasks~\cite{tay2020long}. 
Deep SSMs have reached outstanding results in various domains, including language~\cite{fu2023hungry}, vision~\cite{nguyen2022s4nd} and audio~\cite{goel2022sashimi}.
%biological signals~\cite{gu2021efficiently}, reinforcement learning~\cite{lu2023structured} and online learning~\cite{zucchet2023online}. SSMs also were successfully applied to language modeling and are sometimes used in combination with attention~\cite{fu2023hungry, wang2022pretraining, ma2022mega}. 
At inference time, all SSMs coincide with a stack of linear RNNs, interleaved with position-wise MLPs and normalization layers.
The linearity of the RNNs enables fast parallel processing using FFTs~\cite{gu2021efficiently} or parallel scans~\cite{smith2023simplified}.  
The connection between SSMs and linear RNNs is reinforced by Linear Recurrent Unit (LRU)~\cite{orvieto2023resurrecting} that matches the performance of deep SSMs. While SSMs rely on the discretization of a structured continuous-time latent dynamical system, LRU is directly designed for a discrete-time system.
The main difference between LRU and a standard linear RNN is that LRU operates in the complex domain and its diagonal transition matrix is trained using polar parameterization for stable signal propagation.

\section{Architecture}
\label{sec:arch}

In this section, we present the \model~layer, which is the building unit of our architecture.
We start with some preliminary notations and then describe how our layer computes node representations.
Finally, we analyze its computational complexity. 
%We defer the discussion on expressive power of our module to Section~\ref{sec:expressiveness}.

\paragraph{Preliminaries.}
Let $G = (V, E)$ denote an undirected and unweighted graph,
where $V$ denotes the set of nodes and $E$ denotes the set of edges. 
For any two nodes $v, u \in V$, we use $d(v, u)$ to represent the shortest distance between $v$ and $u$, and we let $d(v, v) = 0$. 
For each target node $v$,
we categorize all other nodes into different hops according to their shortest distances to $v$:
\begin{align}
    \mathcal{N}_k(v) = \{u ~|~ d(v, u) = k\} \quad \text{for}\quad k = 0, 1, \dots, K
\end{align}
where $K$ is the diameter of $G$ or a hyper-parameter specified for the task in hand. 
$\{\mathcal{N}_k(v)\}_{k=1}^K$ can be obtained for every node $v \in V$ by running the Floyd–Warshall algorithm~\cite{floyd1962algorithm, warshall1962theorem} in parallel during data preprocessing and they are saved as masks.

\paragraph{GRED layer.}
The input to the $\ell$-th layer is a set of node representations $\{\!\{\bm{h}_v^{(\ell-1)} \in \mathbb{R}^d ~|~v \in V\}\!\}$.
To compute the output representation $\bm{h}_v^{(\ell)}$ of this layer for a generic target node $v$,
the layer first generates a representation for each set of nodes that share the same shortest distance to $v$~(grey dashed boxes in Figure~\ref{fig:architecture}):
\begin{align}
    \bm{x}_{v, k}^{(\ell)} = \text{AGG} \left( \big\{\!\!\big\{ \bm{h}_u^{(\ell-1)} ~|~ u \in \mathcal{N}_k(v) \big\}\!\!\big\}\right) %\quad \text{for} \quad k = 0, 1, \dots, K
\end{align}
where $\{\!\{\cdot\}\!\}$ denotes a multiset, and AGG is an \textit{injective} multiset function which we parameterize with two wide multi-layer perceptrons (MLPs)\footnote{In practice, with just one hidden layer.}, as usual in the literature~\cite{zaheer2017deep,xu2018powerful}:
%represents a permutation-invariant function that operates on sets~\cite{zaheer2017deep}, which we parametrize, as usual in the literature~(\antonio{cite}), with two MLPs:
\begin{equation}
    \label{eq:deepsets}
    \bm{x}_{v, k}^{(\ell)} = \text{MLP}_2 \left( \sum\nolimits_{u \in \mathcal{N}_k(v)} \text{MLP}_1 \left( \bm{h}_{u}^{(\ell - 1)}\right)\right) \in \mathbb{R}^d.
\end{equation}

These set representations $(\bm{x}_{v, 0}^{(\ell)}, \bm{x}_{v, 1}^{(\ell)}, \dots, \bm{x}_{v, K}^{(\ell)})$ naturally form a sequence according to the
shortest distances.
Then we encode this sequence using a linear RNN:
\begin{align}
    \label{eq:GRB}
    \bm{s}_{v, k}^{(\ell)} = \bm{A}\bm{s}_{v, k-1}^{(\ell)} + \bm{B}\bm{x}_{v, K - k}^{(\ell)} \quad \text{for} \quad k = 0, \ldots, K
\end{align}
where $\bm{s}_{v, k}^{(\ell)} \in \mathbb{R}^{d_s}$ represents the hidden state of the RNN and $\bm{s}_{v, -1}^{(\ell)} = \bm{0}$.
$\bm{A}\in{\mathbb R}^{d_s \times d_s}$ denotes the state transition matrix and $\bm{B} \in \mathbb{R}^{d_s \times d}$ is a matrix to transform the input of the RNN.
Here in~\cref{eq:GRB} the RNN encoding starts from $\bm{x}_{v, K}^{(\ell)}$, proceeds from right to left, and ends at $\bm{x}_{v, 0}^{(\ell)}$,
which corresponds to the signals from distant nodes propagating towards the target node.
The neighborhood hierarchy of the target node $v$ would then be encoded into the final hidden state $\bm{s}_{v, K}^{(\ell)}$ of the RNN.
Note that as in Figure~\ref{fig:architecture}(b), different nodes have different sequences to describe their respective neighborhoods,
and the RNN computations for all nodes can be batched.
Although for a particular target node,
some edges between hop $k$ ($k \geq 1$) and hop $k + 1$ are omitted by converting its neighborhood into a sequence,
those edges would be taken into account for other target nodes.
Therefore,
considering all node representations as a whole,
our model preserves the full graph structural information.
We theoretically prove the expressiveness of the linear RNN and our model in Section~\ref{sec:expressiveness}.

In our model,
we parameterize the linear RNN in a particular diagonal form.
Recall that, over the space of $d_s\times d_s$ non-diagonal real matrices, the set of non-diagonalizable~(in the complex domain) matrices has measure zero~\cite{bhatia2013matrix}. Hence, with probability one over random initializations, $\bm{A}$ is diagonalizable, i.e., $\bm{A} = \bm{V}\bm{\Lambda}\bm{V}^{-1}$, where $\bm{\Lambda} = \text{diag}(\lambda_1, \dots, \lambda_{d_s}) \in \mathbb{C}^{d_s \times d_s}$ gathers the eigenvalues of $\bm{A}$, and columns of $\bm{V}$ are the corresponding eigenvectors. \cref{eq:GRB} is then equivalent to the following diagonal recurrence in the complex domain, up to a linear transformation of the hidden state $\vs$ which can be merged with the output projection $\mW_{\text{out}}$ (\cref{eq:output}):
\begin{align}
    \label{eq:diag}
    \bm{s}_{v, k}^{(\ell)} = \bm{\Lambda}\bm{s}_{v, k - 1}^{(\ell)} + \bm{W}_{\text{in}}\bm{x}_{v, K - k}^{(\ell)}
\end{align}
where $\bm{W}_{\text{in}} = \bm{V}^{-1}\bm{B} \in \mathbb{C}^{d_s \times d}$.
Unrolling the recurrence, we have:
\begin{align}
    \label{eq:diag_full}
    \vs_{v, K}^{(\ell)}=\sum_{k=0}^{K} \mLambda^k\mW_{\text{in}} \vx_{v, k}^{(\ell)}.
\end{align}
\cref{eq:diag_full} can be thought of as a \textit{filter over the hops from the target node} (\cref{fig:filter}),
and the filter weights are determined by the magnitudes of the eigenvalues $\bm{\Lambda}$ and the shortest distances to the target node.
Following the modern literature on deep SSMs~\cite{gupta2022diagonal, gu2022parameterization}, we directly initialize~(without loss of generality) the system in the diagonal form and have $\mLambda$ and $\mW_{\text{in}}$ as trainable parameters\footnote{As done in all state-space models~\cite{gu2021efficiently, smith2023simplified}, we do not optimize over the complex numbers but instead parameterize, for instance, real and imaginary components of $\mW_{\text{in}}$ as real parameters. The imaginary unit $i$ is then used to aggregate the two components in the forward pass.}. 
To guarantee stability (the eigenvalues should be bounded by the unit disk), we adopt the recently introduced LRU initialization~\cite{orvieto2023resurrecting} that parameterizes the eigenvalues with log-transformed polar coordinates.
Through directly learning eigenvalues $\mLambda$,
our model learns to control the influence of signals from distant nodes on the target node, and thus addresses over-squashing caused by iterative 1-hop mixing.
Another advantage of the diagonal linear recurrence is that it can leverage parallel scans~\cite{blelloch1990prefix, smith2023simplified} to avoid computing $\vs$ sequentially on modern hardware.

The output representation $\bm{h}_v^{(\ell)}$ is generated by a non-linear transformation of the last hidden state $\bm{s}_{v, K}^{(\ell)}$:
\begin{equation}
    \label{eq:output}
    \bm{h}_v^{(\ell)} = \text{MLP}_3\left(\Re \left[ \bm{W}_{\text{out}} \bm{s}_{v, K}^{(\ell)} \right] \right)
\end{equation}
where $\bm{W}_{\text{out}} \in \mathbb{C}^{d \times d_s}$ is a trainable weight matrix and
$\Re [\cdot]$ denotes the real part of a complex-valued vector. 
While sufficiently wide MLPs with one hidden layer can parameterize any non-linear mapping, following again the literature on state-space models we choose to place here a gated linear unit (GLU)~\cite{dauphin2017language}: GLU$(\vx) = (\mW_1 \vx)\odot \sigma(\mW_2 \vx)$, with $\sigma$ the sigmoid function and $\odot$ the element-wise product.

The final architecture is composed of stacking several of such layers described above. In practice, we merge $\text{MLP}_1$ in \cref{eq:deepsets} with the non-linear transformation in \cref{eq:output} of the previous layer~(or of the feature encoder) to make the entire architecture more compact.
We add skip connections to both the MLP and the LRU and apply layer normalization to the input of each residual branch.
The overall architecture is illustrated in Figure~\ref{fig:architecture}(a).

\paragraph{Computational complexity.}
For each distance $k$,
the complexity of aggregating the representations of nodes from $\mathcal{N}_k(v)$ for every $v \in V$
is that of one round of message passing, which is $O(|E|)$.
So the total complexity of \cref{eq:deepsets} for all nodes and distances is $O(K|E|)$.
In practice,
since $\{\mathcal{N}_k(v)\}_{k=1}^K$ are pre-computed, \cref{eq:deepsets} for different $k$'s can be performed in parallel to speed up the computation.
The sequential computation of \cref{eq:diag} has total complexity $O(K|V|)$.
However, the linearity of the recurrence and the diagonal state transition matrix enable fast parallel scans to further improve the efficiency.
In the above analysis,
$K$ is upper bounded by the graph diameter, which is usually much smaller than the number of nodes in real-world datasets.
Even in the worst case where the diameter is large,
we can keep the complexity of each layer tractable with a smaller constant $K$ and still access the global information by ensuring the product of model depth and $K$ is no smaller than the diameter.
As a result of the compact and parallelizable architectural design,
our model is highly efficient during training,
as evidenced by our experimental results.

\section{Expressiveness Analysis}
\label{sec:expressiveness}

In this section,
we theoretically analyze the expressive capabilities of the linear RNN (\cref{eq:diag}) and the overall model.
Wide enough linear RNNs have been shown to be able to approximate convolutional filters~\cite{li2022approximation}, and model non-linear dynamic systems when interleaved with MLPs~\cite{orvieto2023universality}.
In the context of this paper,
we are interested in whether the linear RNN can accurately encode the sequence of hop representations (generated by \cref{eq:deepsets})
that describes the neighborhood hierarchy of the target node.
To answer this question,
in the following,
we prove that if the hidden state is large enough,
a linear RNN can express an injective mapping over sequences:

\begin{table*}[t]
    \centering
    \caption{Test classification accuracy (in percent) of our model in comparison with baselines. Performance of baselines is reported by the benchmark~\cite{dwivedi2023benchmarking} or their original papers.
    ``-'' indicates the baseline didn't report its performance on that dataset. We follow the parameter budget $\approx$ 500K.}
    \label{tab:accuracy}
    \vspace{1mm}
    \begin{tabular}{lcccc}
    \toprule
    Model    & MNIST                & CIFAR10                 & PATTERN           & CLUSTER           \\ \midrule
    GCN~\cite{kipf2016semi}     & 90.705$\pm$0.218     & 55.710$\pm$0.381        & \textcolor{darkpink}{\textbf{85.614$\pm$0.032}}  & 69.026$\pm$1.372  \\
    GAT~\cite{velivckovic2017graph}      & 95.535$\pm$0.205     & 64.223$\pm$0.455        & 78.271$\pm$0.186  & 70.587$\pm$0.447  \\
    GIN~\cite{xu2018powerful}      & 96.485$\pm$0.252     & 55.255$\pm$1.527        & 85.590$\pm$0.011  & 64.716$\pm$1.553  \\
    GatedGCN~\cite{bresson2017residual} & \textcolor{darkpink}{\textbf{97.340$\pm$0.143}}     & \textcolor{darkpink}{\textbf{67.312$\pm$0.311}}        & 85.568$\pm$0.088  & \textcolor{darkpink}{\textbf{73.840$\pm$0.326}}  \\ \midrule
    EGT~\cite{hussain2022global}      & 98.173$\pm$0.087     & 68.702$\pm$0.409        & 86.821$\pm$0.020  & 79.232$\pm$0.348  \\
    SAN~\cite{kreuzer2021rethinking}      & - & - & 86.581$\pm$0.037 & 76.691$\pm$0.65 \\ 
    SAT~\cite{chen2022structure}      & - & - & 86.848$\pm$0.037 & 77.856$\pm$0.104 \\
    GPS~\cite{rampavsek2022recipe}      & 98.051$\pm$0.126 &  72.298$\pm$0.356 & 86.685$\pm$0.059& 78.016$\pm$0.180 \\
    Graph MLP-Mixer~\cite{he2023generalization} & \textcolor{darkblue}{\textbf{98.320$\pm$0.040}} & 73.960$\pm$0.330 & - & - \\
    GRIT~\cite{ma2023graph}     & 98.108$\pm$0.111 & \textcolor{darkblue}{\textbf{76.468$\pm$0.881}} & \textcolor{darkblue}{\textbf{87.196$\pm$0.076}} & \textcolor{darkblue}{\textbf{80.026$\pm$0.277}} \\ \midrule
    \model~(Ours) & \textcolor{darkorange}{\textbf{98.383$\pm$0.012}} & \textcolor{darkorange}{\textbf{76.853$\pm$0.185}} & \textcolor{darkorange}{\textbf{86.759$\pm$0.020}} & \textcolor{darkorange}{\textbf{78.495$\pm$0.103}}\\ \bottomrule                               
    \end{tabular}
\end{table*}

\begin{theorem}[Injectivity of linear RNNs]
\label{thm:injection}
Let $\{\vx_v=(\vx_{v,0}, \vx_{v,1}, \vx_{v,2}, \dots, \vx_{v,K_v})~|~v \in V\}$ be a set of sequences (of different lengths $K_v\le K$) of vectors with a (possibly uncountable) set of features $\mathcal{X}\subset\R^d$. Consider a diagonal linear complex-valued RNN with $d_s$-dimensional hidden state, parameters $\mLambda\in\text{diag}(\C^{d_s}), \mW_{in}\in\C^{d_s\times d}$ and recurrence rule $\vs_{v,k} = \mLambda \vs_{v,k-1} + \mW_{in} \vx_{v,K_v-k}$,  initialized at $\vs_{v,-1}={\bm 0}\in\R^{d_s}$ for each $v\in V$. If $d_s\ge (K+1)d$, then there exist $\mLambda, \mW_{in}$ such that the map $R:(\vx_{v,0},\vx_{v,1}, \vx_{v,2}, \dots, \vx_{v,K}) \mapsto \vs_{v,K}$~(with zero right-padding if $K_v<K$) is bijective. Moreover, if the set of RNN inputs has countable cardinality $|\mathcal{X}|=N\le\infty$, then selecting $d_s\ge d$ is sufficient for the existence of an injective linear RNN mapping $R$.
\end{theorem}

The proof can be found in \cref{sec:proof}.
Here we assume zero-padding for $K_v < K$ (for mini-batch training).
If some nodes coincidentally have zero-valued features, we can select a special token which is not in the
dictionary of node features as the padding token. In practice, such an operation is not necessary because node
representations are first fed into an MLP before the linear RNN, which can learn to shift them away from zero.

Based on \cref{thm:injection}, and the well-known conclusion that the parameterization given by \cref{eq:deepsets} can express an injective multiset function~\cite{xu2018powerful},
we have the following corollary:

\begin{corollary}
A wide enough \model~layer is capable of expressing an injective mapping of the list
$(\bm{h}_v, \{\!\!\{\bm{h}_u ~|~ u \in \mathcal{N}_1(v)\}\!\!\}, \{\!\!\{\bm{h}_u ~|~ u \in \mathcal{N}_2(v)\}\!\!\}, \dots, \{\!\!\{\bm{h}_u ~|~ u \in \mathcal{N}_{K_v}(v)\}\!\!\})$ for each $v \in V$.
\end{corollary}

This corollary in turn implies the following result:

%Assuming wide enough architectural components, then the RNN output at any node $v\in V$, in combination with an injective multiset function AGG aggregating neighbors, is an injective function of the list $(\bm{h}_v, \{\!\!\{\bm{h}_u ~|~ u \in \mathcal{N}_1(v)\}\!\!\}, \{\!\!\{\bm{h}_u ~|~ u \in \mathcal{N}_2(v)\}\!\!\}, \dots, \{\!\!\{\bm{h}_u ~|~ u \in \mathcal{N}_{K_v}(v)\}\!\!\})$. 

\begin{corollary}[Expressiveness of \model]
When $K > 1$, one wide enough \model~layer is more expressive than any 1-hop message passing layer.
\end{corollary}

\begin{proof}
We note that $1$-WL assumes an injective mapping of $1$-hop neighborhood, i.e., $(\bm{h}_v, \{\!\!\{\bm{h}_u ~|~ u \in \mathcal{N}_1(v)\}\!\!\})$,
which is a special case of \model~($K = 1$).
When $K>1$, the output of one \model~layer at node $v$, given the injectivity of the linear RNN and AGG, provides a more detailed characterization of its neighborhood than $1$-hop message passing.
This means that if $v$'s $1$-hop neighborhood changes, the output of the \model~layer will also be different.
Therefore, \model~is able to distinguish any two non-isomorphic graphs that are distinguishable by $1$-WL.
Moreover, \model~can distinguish two non-isomorphic graphs which $1$-WL cannot (see Figure~\ref{fig:ex_dist} in the appendix for an example).
\end{proof}

We note that \citeauthor{feng2022how}~\yrcite{feng2022how}~have already proven that multi-hop MPNNs are more expressive than 1-WL,
but are upper bounded by 3-WL, which also applies to our model.
Different from them,
we achieve such expressiveness with a compact and parameter-efficient architecture (i.e., the number of parameters does not increase with $K$),
which is of independent interest and bridges the gap between theory and practice.

\section{Experiments}
In this section,
we evaluate our model on widely used graph benchmarks~\cite{dwivedi2023benchmarking, dwivedi2022long}.
In all experiments,
we train our model using the Adam optimizer with weight decay~\cite{loshchilov2018decoupled} and use the cosine annealing schedule with linear warm-up for the first $5\%$ epochs.
We compare our model against popular MPNNs including GCN~\cite{kipf2016semi}, GAT~\cite{velivckovic2017graph}, 
GIN~\cite{xu2018powerful}, GatedGCN~\cite{bresson2017residual},
and multi-hop MPNN variants~\cite{feng2022how, michel2023path, gutteridge2023drew},
as well as several state-of-the-art graph transformers including Graphormer~\cite{ying2021transformers},
SAT~\cite{chen2022structure}, GPS~\cite{rampavsek2022recipe}, Graph MLP-Mixer~\cite{he2023generalization} and GRIT~\cite{ma2023graph}.
We also measure the training time and memory consumption of \model~to demonstrate its high efficiency.
We use three distinct colors to indicate the performance of our model, the best MPNN, and the best graph transformer.
We detail the hyper-parameters used for our model in the appendix (\cref{tab:hp1}).
In \cref{app:TUDataset},
we validate \model's robustness to over-squashing and compare \model~with SPN~\cite{abboud2022shortest}.

\begin{table}[t]
    \centering
    \caption{Test MAE on ZINC 12K with parameter budget $\approx$ 500K.}
    \label{tab:zinc}
    \vspace{1mm}
    \begin{tabular}{lc}
    \toprule
    Model  & Test MAE $\downarrow$ \\ \midrule
    GCN~\cite{kipf2016semi}    & 0.278$\pm$0.003  \\
    GAT~\cite{velivckovic2017graph}    & 0.384$\pm$0.007  \\
    GIN~\cite{xu2018powerful}    & 0.387$\pm$0.015  \\
    GatedGCN~\cite{bresson2017residual}     &  0.282$\pm$0.015 \\ 
    PNA~\cite{corso2020principal}       &  0.188$\pm$0.004 \\ 
    KP-GIN~\cite{feng2022how} & 0.093$\pm$0.007 \\
    PathNN~\cite{michel2023path}   & \textcolor{darkpink}{\textbf{0.090$\pm$0.004}} \\ \midrule
    SAN~\cite{kreuzer2021rethinking}          &  0.139$\pm$0.006 \\
    Graphormer~\cite{ying2021transformers}   &  0.122$\pm$0.006 \\
    K-subgraph SAT~\cite{chen2022structure}   & 0.094$\pm$0.008 \\  
    GPS~\cite{rampavsek2022recipe} & 0.070$\pm$0.004 \\ 
    Graph MLP-Mixer~\cite{he2023generalization} & 0.073$\pm$0.001 \\
    GRIT~\cite{ma2023graph}       & \textcolor{darkblue}{\textbf{0.059$\pm$0.002}} \\ \midrule
    \model~(Ours) & \textcolor{darkorange}{\textbf{0.077$\pm$0.002}} \\
    \bottomrule
    \end{tabular}
\end{table}

\paragraph{Benchmarking GNNs.}
We first evaluate our model on the node classification datasets: PATTERN and CLUSTER, and graph classification datasets: MNIST and CIFAR10 from~\cite{dwivedi2023benchmarking}.
To get the representation for the entire graph,
we simply do average pooling over all node representations.
Our model doesn't use any positional encoding.
We train our model four times with different random seeds and report the average accuracy with standard deviation.
The results are shown in Table~\ref{tab:accuracy}.
From the table we see that graph transformers generally perform better than MPNNs.
Among the four datasets,
PATTERN models communities in social networks and all nodes are reachable within 3 hops,
which we conjecture is why the performance gap between graph transformers and MPNNs is only marginal.
For a more difficult task, like CIFAR10, that requires information from a relatively larger neighborhood,
graph transformers work more effectively.
\model~performs well on all four datasets
and consistently outperforms MPNNs.
Notably,
on MNIST and CIFAR10,
\model~achieves the best accuracy, outperforming state-of-the-art models Graph MLP-Mixer and GRIT,
which validates that our model can effectively aggregate information beyond the local neighborhood.

\paragraph{ZINC 12K.} Next, we report the test MAE of our model on ZINC 12K~\cite{dwivedi2023benchmarking}.
The average MAE and standard deviation of four runs with different random seeds are shown in Table~\ref{tab:zinc} 
along with baseline performance from their original papers.
From Table~\ref{tab:zinc} we can observe that the performance of our model is significantly better than 
that of existing MPNNs.
In particular,
\model~outperforms other multi-hop MPNN variants~\cite{feng2022how, michel2023path},
which shows our architecture is more effective in aggregating multi-hop information.
Comparing \model~with graph transformers,
we find that it outperforms several graph transformer variants (SAN, Graphormer, and K-subgraph SAT)
and approaches the state-of-the-art model.
This is impressive given that our model doesn't require any positional encoding.
These results evidence that our model can encode graph structural information through the natural inductive bias of recurrence.

\begin{table}[t]
\centering
\caption{Test performance on Peptides-func/struct.}
\label{tab:lrgb}
\vspace{1mm}
\begin{tabular}{lcc}
\toprule
\multirow{2}{*}{Model} & Peptides-func              & Peptides-struct              \\ 
                        & Test AP $\uparrow$         & Test MAE $\downarrow$        \\ \midrule
 GCN$^*$                    & 0.6860$\pm$0.0050          & \textcolor{darkpink}{\textbf{0.2460$\pm$0.0007}}            \\
 GINE$^*$                   & 0.6621$\pm$0.0067          & 0.2473$\pm$0.0017            \\
 GatedGCN$^*$               & 0.6765$\pm$0.0047          & 0.2477$\pm$0.0009            \\ 
 PathNN                     & 0.6816$\pm$0.0026 & 0.2540$\pm$0.0046 \\
 DRew                       & 0.6996$\pm$0.0076 & 0.2781$\pm$0.0028 \\
 DRew+LapPE                 & \textcolor{darkpink}{\textbf{0.7150$\pm$0.0044}} & 0.2536$\pm$0.0015 \\ \midrule
 %GatedGCN+RWSE          & \textcolor{darkpink}{\textbf{0.6069$\pm$0.0035}}          & \textcolor{darkpink}{\textbf{0.3357$\pm$0.0006}}            \\ \midrule
 %Transformer+LapPE      & 0.6326$\pm$0.0126          & 0.2529$\pm$0.0016            \\
 SAN+LapPE              & 0.6384$\pm$0.0121          & 0.2683$\pm$0.0043            \\
 GPS                    & 0.6535$\pm$0.0041          & 0.2500$\pm$0.0005            \\
 Graph-MLPMixer & 0.6970$\pm$0.0080 & 0.2475$\pm$0.0015 \\
 GRIT                   & \textcolor{darkblue}{\textbf{0.6988$\pm$0.0082}} & \textcolor{darkblue}{\textbf{0.2460$\pm$0.0012}}            \\ \midrule
 \model~(Ours)          & \textcolor{darkorange}{\textbf{0.7085$\pm$0.0027}}          & \textcolor{darkorange}{\textbf{0.2503$\pm$0.0019}}            \\ 
 \model+LapPE           & \textcolor{darkorange}{\textbf{0.7133$\pm$0.0011}}          & \textcolor{darkorange}{\textbf{0.2455$\pm$0.0013}}            \\ \bottomrule
 \end{tabular}
 \end{table}

\paragraph{Long Range Graph Benchmark.}
To further test the long-range modeling capability of \model,
we evaluate it on the Peptides-func and Peptides-struct datasets from~\cite{dwivedi2022long}.
We follow the 500K parameter budget and train our model four times with different random seeds.
The results are displayed in Table~\ref{tab:lrgb}.
The performance of GCN, GINE and GatedGCN (marked with $*$) comes from a recent report~\cite{tonshoff2023did} that extensively tuned their hyper-parameters \textit{with} positional encoding.
Performance of other baselines is reported by respective papers.
We can observe that,
even without positional encoding,
\model~significantly outperforms all baselines except DRew+LapPE on Peptides-func,
and its performance on Peptides-struct also matches that of the best graph transformer.
Note that on Peptides-struct, DRew+LapPE performs worse than \model.
These results demonstrate the strong long-range modeling capability of our architecture itself.
As a supplement,
we test \model+LapPE by concatenating Laplacian positional encoding with node features,
and we find it slightly improves the performance.
We leave the combination of more advanced positional encoding with \model~to future work.

\begin{figure}[t]
 \centering
 \includegraphics[width=0.7\linewidth]{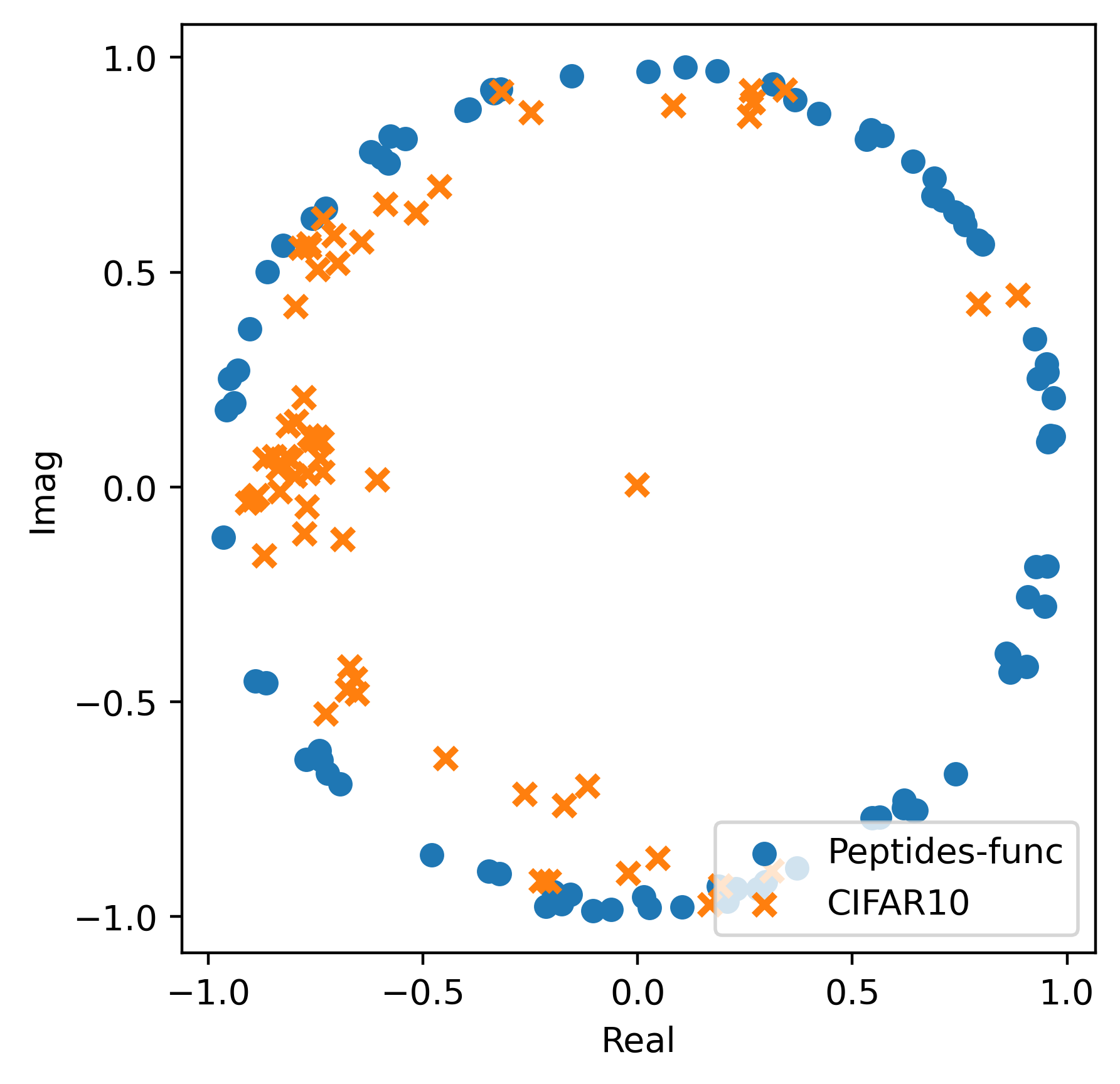}
 \caption{Learned (complex) eigenvalues of the first \model~layer on CIFAR10 and Peptides-func.}
 \label{fig:eig}
\end{figure}

To illustrate how \model~can learn to preserve long-range information,
we examine the eigenvalues learned by the linear RNN (i.e., $\mLambda$ in \cref{eq:diag}) after training,
as shown in \cref{fig:eig}.
We observe from the figure that the eigenvalues are pushed close to $1$ for the long-range task Peptides-func,
which prevent the signals of distant nodes from decaying too fast.
Compared with Peptides-func,
CIFAR10 requires the model to utilize more information from the local neighborhood,
so the magnitudes of the eigenvalues become smaller.

\begin{table}[h]
\centering
\caption{Average training time per epoch and GPU memory consumption for GRIT and \model.}
\label{tab:speed}
\vspace{-2mm}
\begin{tabular}{cccc}\\\toprule  
Model & ZINC 12K     & CIFAR10 & Peptides-func \\ \midrule
\multirow{2}{*}{GRIT}  & 23.9s    & 244.4s  & 225.6s  \\
                       & 1.9GB    & 4.6GB   & 22.5GB  \\ \midrule
\multirow{2}{*}{\model}  & 3.7s    & 27.8s   & 158.9s    \\
                                & 1.5GB   & 1.4GB   & 18.5GB    \\ \midrule
Speedup & \textbf{6.5$\times$}& \textbf{8.8$\times$} &  \textbf{1.4$\times$} \\ \bottomrule
\end{tabular}
\end{table}

\begin{figure}[h]
    \centering
    \includegraphics[width=0.48\textwidth]{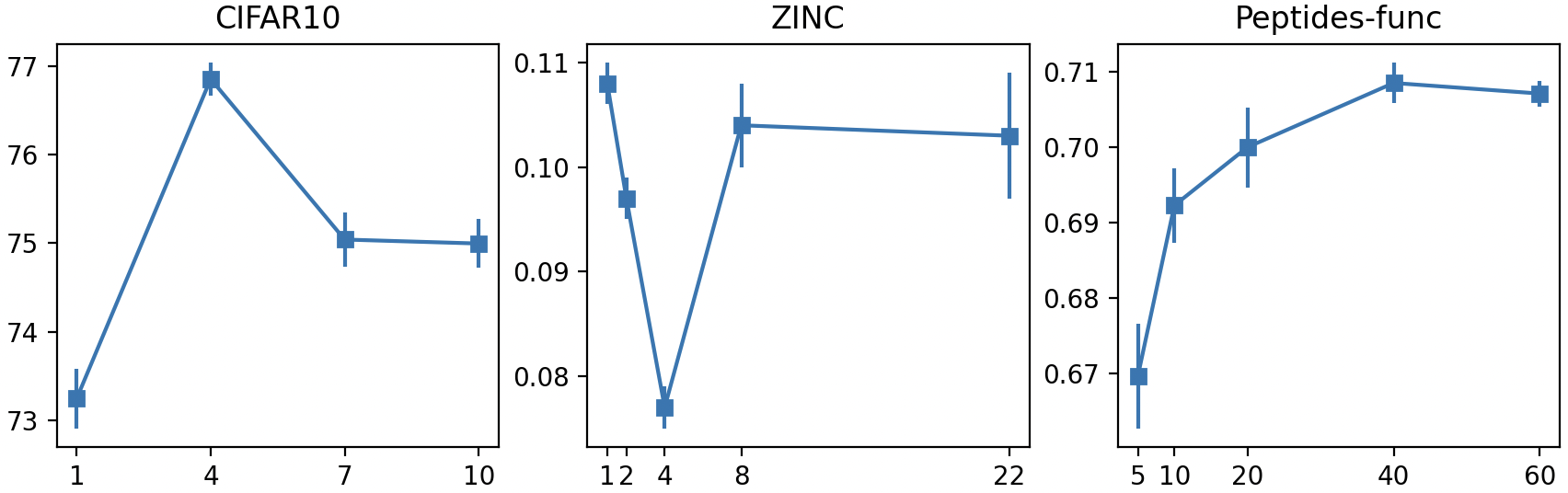}
    \caption{Effect of $K$ on performance.}
    \label{fig:K_ablation}
\end{figure}

\paragraph{Training efficiency.}
To demonstrate the high efficiency of our model,
we record the average training time per epoch and GPU memory consumption on ZINC, CIFAR10 and Peptides-func. 
We compare our measurements with those of the state-of-the-art graph transformer GRIT.
Both models are trained using the same batch size and on
a single RTX A5000 GPU with 24GB memory.
As shown in \cref{tab:speed},
our model improves the training efficiency by a huge margin,
which stems from our compact and parallelizable architecture design.

\paragraph{Effect of $K$ on performance.}
Recall that the length of recurrence $K$ can be regarded as a hyper-parameter in \model. 
In \cref{fig:K_ablation}, we show how different $K$ values affect the performance of \model~on CIFAR10, ZINC and Peptides-func,
keeping the depth and hidden dimension of the architecture unchanged (without positional encoding).
On CIFAR10 and ZINC,
while directly setting $K$ as the diameter already outperforms classical MPNNs,
we find that the optimal $K$ value that yields the best performance lies strictly between 1 and the diameter.
This may be because information that is too far away is less important for these two tasks (interestingly, the best $K$ value for CIFAR10 is similar to the width of a convolutional kernel on a normal image).
On Peptides-func,
the performance is more monotonic with $K$.
When $K=40$,
\model~outperforms the best graph transformer GRIT.
We observe no further performance gain on Peptides-func when we increase $K$ to 60.

\begin{figure}[h]
    \centering
    \includegraphics[width=0.48\textwidth]{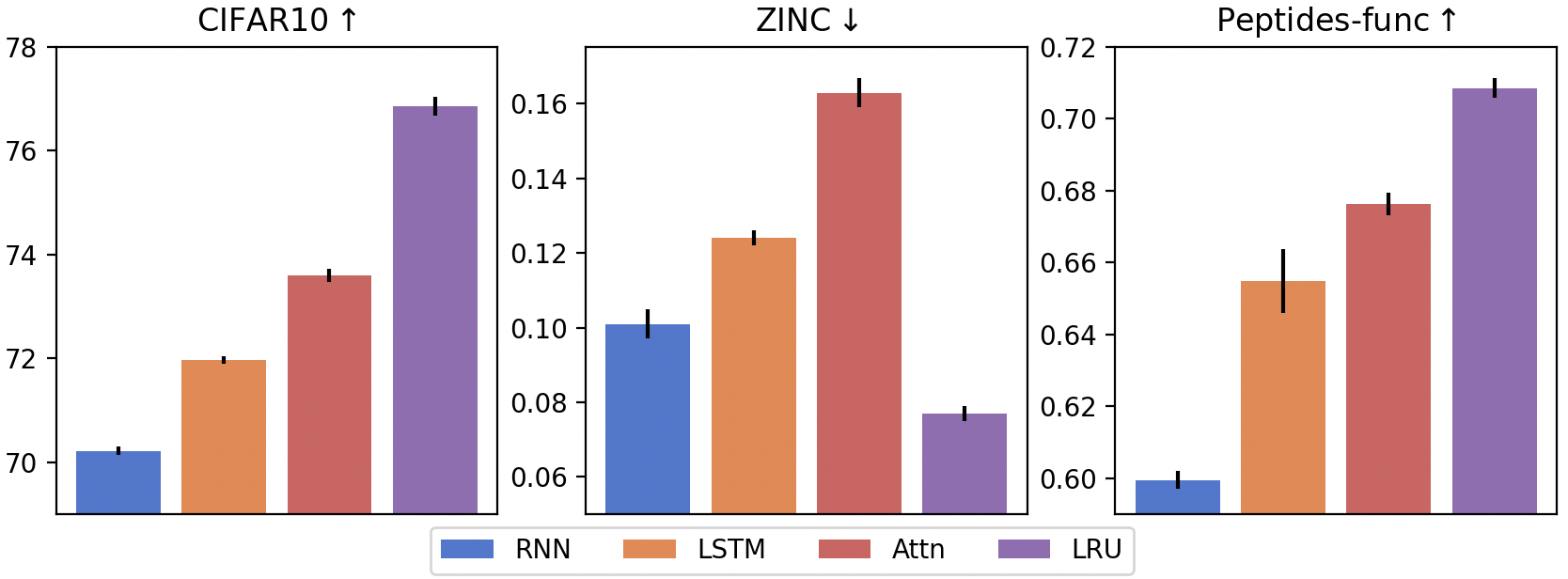}
    \caption{Performance of \model~using RNNs of different flavors.}
    \label{fig:rnn_ablation}
\end{figure}

\paragraph{Comparing RNNs of different flavors.}
Finally, we highlight the necessity of the LRU component (\cref{eq:diag}) of \model~by replacing it with a vanilla RNN,
a standard LSTM cell or $8$-head self-attention.
The performance of different variants is shown in \cref{fig:rnn_ablation}.
We use the same number of layers and $K$ for all models and tune the learning rate, weight decay and dropout rate in the same grid.
None of the variants use positional encoding.
We can observe that \model$_\text{LSTM}$ performs better than \model$_\text{RNN}$ on CIFAR10 and Peptides-func.
Since LSTM can alleviate the training instability of the vanilla RNN,
the improvement of \model$_\text{LSTM}$ over \model$_\text{RNN}$ is particularly large on the long-range dataset Peptides-func.
\model$_\text{Attn}$ allows direct interaction with each hop and thus also yields good performance on Peptides-func.
However, self-attention cannot provide good inductive bias because it cannot model the order of the hop sequence, which can explain why the performance of \model$_\text{Attn}$ is the worst on ZINC.
\model$_\text{LRU}$ consistently outperforms the other variants,
attributed to its advanced parameterization for stable signal propagation and great expressive power.

\section{Conclusion}

In this paper, we introduce the Graph Recurrent Encoding by Distance (GRED) model for graph representation learning. By integrating permutation-invariant neural networks with linear recurrent neural networks, GRED effectively harnesses information from distant nodes without the need for positional encoding or computationally expensive attention mechanisms. 
Theoretical and empirical evaluations confirm \model's superior performance compared with existing MPNNs and highly competitive results compared with state-of-the-art graph transformers at a higher training efficiency. This positions \model~as a powerful, efficient, and promising model for graph representation learning.

\section*{Acknowledgements}
We thank the anonymous reviewers for their valuable feedback,
which helped us improve the paper.
Antonio Orvieto acknowledges the financial support of the Hector Foundation.
Yuhui Ding would like to personally thank Jiaxin Zhang for her support during the stressful time before the deadline.

\section*{Impact Statement}

This paper presents work whose goal is to advance the field of 
Machine Learning. There are many potential societal consequences 
of our work, none of which we feel must be specifically highlighted here.


\bibliography{icml}
\bibliographystyle{icml2024}

%%%%%%%%%%%%%%%%%%%%%%%%%%%%%%%%%%%%%%%%%%%%%%%%%%%%%%%%%%%%%%%%%%%%%%%%%%%%%%%
%%%%%%%%%%%%%%%%%%%%%%%%%%%%%%%%%%%%%%%%%%%%%%%%%%%%%%%%%%%%%%%%%%%%%%%%%%%%%%%
% APPENDIX
%%%%%%%%%%%%%%%%%%%%%%%%%%%%%%%%%%%%%%%%%%%%%%%%%%%%%%%%%%%%%%%%%%%%%%%%%%%%%%%
%%%%%%%%%%%%%%%%%%%%%%%%%%%%%%%%%%%%%%%%%%%%%%%%%%%%%%%%%%%%%%%%%%%%%%%%%%%%%%%
\newpage
\appendix
\onecolumn
\section{Proof of \cref{thm:injection}}
\label{sec:proof}

\begin{proof}[Proof]

For now, let us assume for ease of exposition that all sequences are of length $K$.
Also, let us, for simplicity, omit the dependency on $v\in V$ and talk about generic sequences.

The proof simply relies on the idea of writing the linear recurrence in matrix form~\cite{gu2022parameterization, orvieto2023universality}. Note that for a generic input $\vx = (\vx_0,\vx_1, \vx_2,\dots, \vx_K)\in\R^{d\times (K+1)}$, the recurrence output can be rewritten in terms of powers of $\mLambda= \diag(\lambda_1,\lambda_2,\dots, \lambda_{d_s})$ as follows:
    \begin{equation}
        \vs_K =\sum_{k=0}^{K} \mLambda^k\mW_{\text{in}} \vx_{k}.
    \end{equation}    
    We now present sufficient conditions for the map $R:(\vx_0, \vx_1, \vx_2, \dots, \vx_K)\mapsto \vs_K$ to be injective or bijective. The proof for bijectivity does not require the set of node features to be in a countable set, and it is simpler. 
    
    \paragraph{Bijective mapping.} First, let us design a proper matrix $\mW_{\text{in}}\in\R^{d_s\times d}$. We choose $d_s=(K+1)d$ and set $\mW_{\text{in}} = \mI_{d\times d}\otimes {\bf{1}}_{(K+1)\times 1}$. As a result, the RNN will independently process each dimension of the input with a sub-RNN of size $(K+1)$. The resulting $\vs_K\in\R^{(K+1) d}$ will gather each sub-RNN output by concatenation. We can then restrict our attention to the first dimension of the input sequence:
    \begin{equation}
    (\vs_K)_{1:(K+1)} =\sum_{k=0}^{K} \diag(\lambda_1,\lambda_2,\dots, \lambda_{K+1})^k {\bf{1}}_{(K+1)\times 1} x_{k, 1}.
    \end{equation}
    This sum can be written conveniently by multiplication using a Vandermonde matrix:
    \begin{equation}
    (\vs_K)_{1:(K+1)} = \begin{pmatrix}
    \lambda_1^{K}&\lambda_1^{K-1} &\cdots& \lambda_1&1\\
    \lambda_2^{K}&\lambda_2^{K-1} &\cdots& \lambda_2&1\\
    \vdots&\vdots &\ddots&\vdots&\vdots\\
    \lambda_{K+1}^{K}&\lambda_{K+1}^{K-1} &\cdots& \lambda_{K+1}&1\\
    \end{pmatrix} \vx_{0:K, 1}^{\leftarrow}.
    \end{equation}
    where $\vx_{0:K, 1}^{\leftarrow}$ is the input sequence (first dimension) in reverse order.
    The proof is concluded by noting that Vandermonde matrices of size $(K+1)\times (K+1)$ are full-rank since they have non-zero determinant $\prod_{1\le i< j\le (K+1)}(\lambda_i-\lambda_j)\ne 0$, under the assumption that all $\lambda_i$ are distinct. Note that one does not need complex eigenvalues to achieve this, both $\mLambda$ and $\mW_{\text{in}}$ can be real. However, as discussed by~\citeauthor{orvieto2023universality}~\yrcite{orvieto2023universality}, complex eigenvalues improve conditioning of the Vandermonde matrix.

    \paragraph{Injective mapping.} The condition for injectivity is that if $\vx\ne\hat \vx$, then $R(\vx)\ne R(\hat \vx)$. In formulas,
    \begin{equation}
        \vs_K - \hat \vs_K =\sum_{k=0}^{K} \mLambda^k\mW_{\text{in}} (\vx_{k}-\hat\vx_{k})\ne \bm{0}
    \end{equation}
    Let us assume the state dimension coincides with the input dimension, and let us set $\mW_{\text{in}}=\mI_{d\times d}$. Then, we have the condition:
    \begin{equation}
        \vs_K - \hat \vs_K =\sum_{k=0}^{K} \mLambda^k(\vx_{k}-\hat\vx_{k})\ne \bm{0}.
    \end{equation}
    Since $\mLambda = \diag(\lambda_1,\lambda_2,\dots, \lambda_{d})$ is diagonal, we can study each component of $\vs_K - \hat \vs_K$ separately. We therefore require
    \begin{equation}
        s_{K,i} - \hat s_{K, i} =\sum_{k=0}^{K} \lambda^k_i(x_{k, i}-\hat x_{k, i})\ne 0\qquad \forall i\in\{1,2,\dots, d\}.
    \end{equation}
    We can then restrict our attention to linear one-dimensional RNNs~(i.e. filters) with one-dimensional input $\vx\in\R^{1\times (K+1)}$. We would like to choose $\lambda\in\C$ such that 
    \begin{equation}
        \sum_{k=0}^{K} \lambda^k(x_{k}-\hat x_{k})\ne 0
    \end{equation}
    Under the assumption $|\mathcal{X}|=N\le\infty$, $\vx-\bar \vx$ is a generic signal in a countable set~($N(N-1)/2 =\Omega(N^2)$ possible choices). Let us rename $\vz := \vx-\bar \vx\in \mathcal{Z}\subset\R^{1\times (K+1)}$, $|\mathcal{Z}| = \Omega(N^2)$. We need
    \begin{equation}
        \langle \bar \vlambda, \vz\rangle \ne 0,\qquad \forall \vz\in\mathcal{Z},\qquad \text{where} \quad \bar \vlambda = (1, \lambda, \lambda^2,\cdots, \lambda^K)
    \end{equation}

Such $\lambda$ can always be found \textit{in the real numbers}, and the reason is purely geometric. We need
$$\bar \vlambda\notin \mathcal{Z}_\perp := \bigcup_{\vz\in\mathcal{Z}}\vz_\perp.$$
Note that $\dim(\vz_\perp) = K$, so $\dim(\mathcal{Z}_\perp) = K$ due to the countability assumption --- in other words the Lebesgue measure vanishes: $\mu(\mathcal{Z}_\perp; \R^{K+1})=0$. If $\bar \vlambda$ were an arbitrary vector, we would be done since we can pick it at random and with probability one $\bar \vlambda\notin \mathcal{Z}_\perp$. But $\bar \vlambda $ is structured~(lives on a 1-dimensional manifold), so we need one additional step.

\begin{figure}
    \centering
    \includegraphics[width=0.4\textwidth]{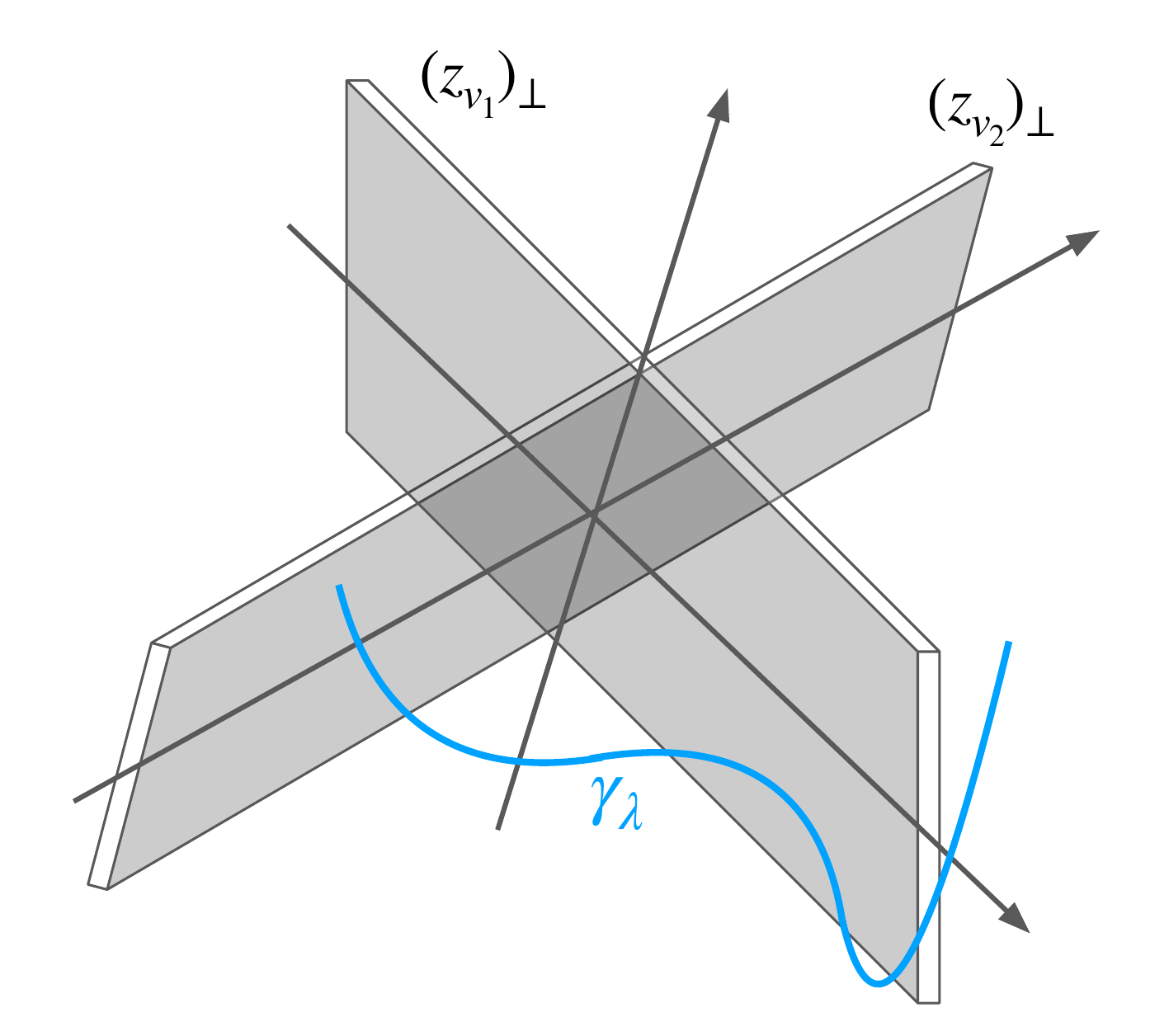}
    \caption{Proof illustration for \cref{thm:injection}. The set $\mathcal{Z}_\perp$ is depicted as union of hyperplanes, living in $\R^{K+1}$ and here sketched in three dimensions. The curve $\gamma_\lambda: \lambda \mapsto (1, \lambda, \lambda^2,\cdots, \lambda^K)$ is shown as a blue line. The proof shows that, for $\lambda\in\R$, the support of $\gamma_\lambda$ is not entirely contained in $\mathcal{Z}_\perp$.}
    \label{fig:gamma_a}
\end{figure}

Note that $\bar \vlambda$ is parametrized by $\lambda$, and in particular $\mathbb{R}\ni \lambda\mapsto \bar \vlambda \in\mathbb{R}^{K+1}$ is a curve in $\mathbb{R}^{K+1}$, we denote this as $\gamma_\lambda$. Now, crucially, note that the support of $\gamma_\lambda$ is a smooth curved manifold for $K>1$. In addition, crucially, $\bm{0} \notin \gamma_\lambda$. We are done: it is impossible for the $\gamma_\lambda$ curve to live in a $K$ dimensional space composed of a union of hyperplanes; it indeed has to span the whole $\mathbb{R}^{K+1}$, without touching the zero vector~(see Figure~\ref{fig:gamma_a}). The reason why it spans the whole  $\mathbb{R}^{K+1}$ comes from the Vandermonde determinant! Let $\{\lambda_1,\lambda_2,\cdots, \lambda_{K+1}\}$ be a set of $K+1$ distinct $\lambda$ values. The Vandermonde matrix
\begin{equation*}
    \begin{pmatrix}
    \lambda_1^{K}&\lambda_1^{K-1} &\cdots& \lambda_1&1\\
    \lambda_2^{K}&\lambda_2^{K-1} &\cdots& \lambda_2&1\\
    \vdots&\vdots &\ddots&\vdots&\vdots\\
    \lambda_{K+1}^{K}&\lambda_{K+1}^{K-1} &\cdots& \lambda_{K+1}&1\\
    \end{pmatrix}
\end{equation*}
has determinant $\prod_{1\le i< j\le (K+1)}(\lambda_i-\lambda_{j})\ne 0$ --- it's full rank, meaning that the vectors $\bar \vlambda_1, \bar \vlambda_2,\dots, \bar \vlambda_{K+1}$ span the whole $\mathbb{R}^{K+1}$. Note that $\lambda\mapsto \bar \vlambda$ is a continuous function, so even though certain $\bar \vlambda_i$ might live on $\mathcal{Z}_\perp$ there exists a value in between them which is not contained in $\mathcal{Z}_\perp$. 
\end{proof}

\begin{figure}[h]
    \centering
    \includegraphics[width=0.5\textwidth]{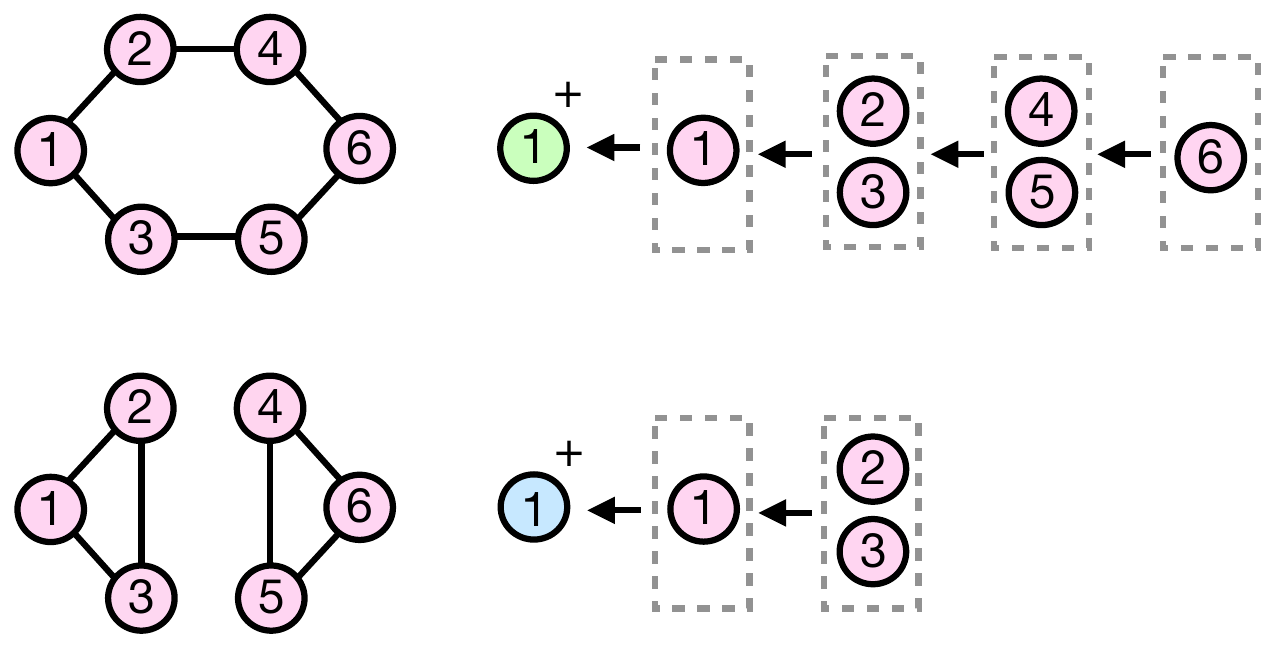}
    \caption{\model~provides distinct updates for the two graphs above. Such graphs, however, are indistinguishable by the 1-WL isomorphism test, assuming~(worst-case) nodes have identical features.}
    \label{fig:ex_dist}
\end{figure}

\begin{table}[h]
    \centering
    \caption{Hyper-parameters for \model. For PATTERN and CLUSTER, $K$ is the diameter of the graph. For \model+LapPE in \cref{tab:lrgb},
the Laplacian PE uses the 10 smallest eigenvectors and a hidden dimension of 16.}
    \label{tab:hp1}
    \vspace{1mm}
    \begin{tabular}{cccccccc}
    \toprule
    Hyper-parameter & ZINC 12K & MNIST & CIFAR10 & PATTERN & CLUSTER  & Peptides-func & Peptides-struct\\ \midrule
    Layers      & 11       & 4     & 8       & 10      & 16       & 8 & 4 \\
    $K$            & 4        & 2     & 4       & -       & -        & 40 & 4 \\
    Dropout        & 0.2      & 0.15  & 0.15    & 0.2     & 0.2      & 0.2 & 0.2\\
    $d$            & 72       & 128   & 96     & 64      & 64       & 88 & 128 \\
    $d_s$          & 72       & 96   & 64     & 64      & 64       & 88 & 96 \\
    Learning rate  & 0.001    & 0.001 & 0.001   & 0.001   & 0.001    & 0.001 & 0.001 \\
    Weight decay   & 0.1      & 0.1   & 0.1     & 0.1     & 0.2      & 0.2 & 0.2 \\
    Epochs      & 2000     & 600   & 600     & 100     & 100      & 200 & 200 \\
    Batch size     & 32       & 16    & 16     & 32      & 32       & 32 & 32 \\ \bottomrule
    \end{tabular}
\end{table}

\section{Additional Results}
\label{app:TUDataset}
To validate the robustness of \model~to over-squashing,
we consider the Tree-NeighborsMatch task proposed by~\citeauthor{alon2020bottleneck}~\yrcite{alon2020bottleneck}.
Following the same experimental setup as \citeauthor{alon2020bottleneck}~\yrcite{alon2020bottleneck},
we report the training accuracy of \model~in \cref{tab:tree_neighbor_match} to show how well \model~can harness long-range information to fit the data.
As a comparison, we quote the performance of GIN which uses the same multiset aggregation as \model.
For GIN, a network with $r+1$ layers is trained for each tree depth in the original paper \cite{alon2020bottleneck}, while for \model~the number of layers is only around half of the tree depth, with an appropriate $K > 1$ to avoid under-reaching.
Over-squashing starts to affect GIN at $r=4$, preventing the model from effectively using distant information to perfectly fit the data. On the contrary, \model~is not affected by over-squashing across different tree depths.

We further evaluate \model~on NCI1 and PROTEINS from TUDataset.
We follow the experimental setup of SPN~\cite{abboud2022shortest}, and report the average test accuracy and standard deviation across $10$ train/val/test splits, as shown in Table~\ref{tab:tud}.
We use the same $K$ for \model~as for SPN and cite the performance reported by the SPN paper~\cite{abboud2022shortest}.
Our model generalizes well to TUDataset and shows good performance.
Furthermore,
\model~outperforms SPN~\cite{abboud2022shortest} with the same number of hops, which verifies that \model~is a better architecture for aggregating large neighborhoods.

\begin{minipage}{.5\textwidth}
    \centering
    \captionof{table}{Accuracy across tree depths.}
    \vspace{1mm}
    \label{tab:tree_neighbor_match}
    \begin{tabular}{cccccccc}
    \toprule
    Model  & $r=2$ & $3$ & $4$ & $5$ & $6$ & $7$ & $8$ \\ \midrule
    GIN    & 1.0   & 1.0 & 0.77 & 0.29 & 0.20 & - & - \\
    \model & 1.0   & 1.0 & 1.0  & 1.0  & 1.0  & 1.0 & 0.95 \\ \bottomrule
    \end{tabular}
\end{minipage}
\begin{minipage}{.49\textwidth}
    \centering
    \captionof{table}{Performance (accuracy) of \model~on TUDataset.}
    \vspace{1mm}
    \label{tab:tud}
    \begin{tabular}{lcc}
    \toprule
    Model &  NCI1          & PROTEINS \\ \midrule
    DGCNN &  76.4$\pm$1.7  & 72.9$\pm$3.5 \\ 
    DiffPool & 76.9$\pm$1.9 & 73.7$\pm$3.5 \\
    ECC      & 76.2$\pm$1.4 & 72.3$\pm$3.4 \\
    GIN      & 80.0$\pm$1.4 & 73.3$\pm$4.0 \\
    GraphSAGE & 76.0$\pm$1.8 & 73.0$\pm$4.5 \\
    SPN ($K=10$) & 78.2$\pm$1.2 & 74.5$\pm$3.2 \\
    \model~($K=10$) & \textbf{82.6$\pm$1.4} & \textbf{75.0$\pm$2.9} \\
    \bottomrule
    \end{tabular}
\end{minipage}

%%%%%%%%%%%%%%%%%%%%%%%%%%%%%%%%%%%%%%%%%%%%%%%%%%%%%%%%%%%%%%%%%%%%%%%%%%%%%%%
%%%%%%%%%%%%%%%%%%%%%%%%%%%%%%%%%%%%%%%%%%%%%%%%%%%%%%%%%%%%%%%%%%%%%%%%%%%%%%%

\end{document}